\documentclass[11pt]{article}
\usepackage{appendix}
\usepackage[figuresright]{rotating}
\usepackage{amsmath,amssymb,amscd,amsthm}
\usepackage{graphicx}
\usepackage{dcolumn}
\usepackage{bm}
\usepackage{ifthen}
\usepackage{rays_defs_18}

\usepackage{hyperref}
\hypersetup{
    bookmarks=true,         
    unicode=false,          
    pdftoolbar=true,        
    pdfmenubar=true,        
    pdffitwindow=false,     
    pdfstartview={FitH},    
    pdfauthor={Reinhard Heckel},     
    pdfsubject={Subject},   
    pdfcreator={Reinhard Heckel},   
    pdfproducer={Producer}, 
    pdfnewwindow=true,      
    colorlinks=true,       
    linkcolor=red,          
    citecolor=green,        
    filecolor=magenta,      
    urlcolor=cyan           
}

\usepackage[
backend=bibtex,
url=false,isbn=false,doi=false,
date=year,
hyperref=auto,
style=alphabetic,
natbib=true,
sorting=nty,
firstinits=true,
maxnames=2, maxbibnames=10,
]{biblatex}

\bibliography{./bibliography.bib} 

\usepackage{comment}
\usepackage{tikz}
\usepackage{pgfplots}
\usepgfplotslibrary{groupplots} 
\usetikzlibrary{pgfplots.groupplots}

\usepackage{pgfplots,pgfplotstable}
\usepgfplotslibrary{fillbetween}
\pgfplotsset{compat=1.10}
\usepackage{amsmath,amssymb,ifthen,color,framed,bm}

\usepackage{colortbl}
\definecolor{tblblue}{RGB}{101,124,191}
\definecolor{tblred}{rgb}{1,0.93,0.93}
\definecolor{DarkBlue}{rgb}{0,0,0.7} 
\definecolor{BrickRed}{RGB}{203,65,84}

\usepackage{placeins}

\newtheorem{lemma}{Lemma}
\newtheorem{definition}{Definition}

\newtheorem{theorem}{Theorem}

\newtheorem{proposition}{Proposition}

\newcommand{\clr}[1]{{#1}}

\newcommand\relu{\mathrm{relu}}
\newcommand\cn{\mathrm{cn}}

\renewcommand\diag{\mathrm{diag}}
\newcommand\loss{L}
\newcommand\setU{\mathcal U}

\newcommand\kout{k_{\mathrm{out}}}
\newcommand\prior{G}
\newcommand\param{\mC}

\newcommand\img{\vx}
\newcommand\noise{\eta}
\newcommand\Nparam{N}

\newcommand\setN{\mathcal N}
\newcommand\setB{\mathcal B}
\newcommand\setG{\mathcal G}

\newcommand\empmean{\mathrm{mean}}
\newcommand\empvar{\mathrm{var}}

\begin{document}

\begin{center}

{\bf{\LARGE{
Regularizing linear inverse problems with 
convolutional neural networks
}}}

\vspace*{.2in}

{\large{
\begin{tabular}{cccc}
Reinhard Heckel\\
\end{tabular}
}}

\vspace*{.05in}

\begin{tabular}{c}
Dept. of Electrical and Computer Engineering, Technical University of Munich \\
\end{tabular}

\vspace*{.1in}

\today

\vspace*{.1in}

\end{center}


\begin{abstract}
Deep convolutional neural networks trained on large datsets have emerged as an intriguing alternative for compressing images and solving inverse problems such as denoising and compressive sensing. 
However, it has only recently been realized that even without training, convolutional networks can function as concise image models, and thus regularize inverse problems. 
In this paper, we provide further evidence for this finding by studying variations of convolutional neural networks that map few weight parameters to an image. 
The networks we consider only consist of convolutional operations, with either fixed or parameterized filters followed by ReLU non-linearities. 
We demonstrate that with both fixed and parameterized convolutional filters those networks enable representing images with few coefficients. 
What is more, the underparameterization enables regularization of inverse problems, in particular recovering an image from few observations. 
We show that, similar to standard compressive sensing guarantees, on the order of the number of model parameters many measurements suffice for recovering an image from compressive measurements.
Finally, we demonstrate that signal recovery with a un-trained convolutional network 
outperforms standard $\ell_1$ and total variation minimization for magnetic resonance imaging (MRI). 
\end{abstract}

\section{Introduction}
In this paper, we consider the problem of recovering an unknown signal $\vx^\ast \in \reals^n$ from few noisy measurements
\[
\vy = f(\vx^\ast) + \noise 
 \quad \in \reals^m
\]
where $f$ is a known measurement operator and $\noise$ is additive noise. We focus on the classical compressive sensing problem, where $f(\vx) = \mA \vx$ is a linear measurement operator. 
Since the number of measurements, $m$, is smaller than the dimension of the image $\vx^\ast$, prior assumption in form of an image model are required to regularize the inverse problem of reconstructing the signal $\vx^\ast$ from the measurement $\vy$.

Image models play a central role in practically every image-related application in signal processing, computer vision, and machine learning. 
An image model captures low-dimensional structure of natural images, which in turn enables efficient image recovery or processing. 
Image models have continuously developed from classical handcrafted models such as
overcomplete bases, wavelets, and sparse representations~\cite{mallat_wavelet_2008} to learned image representations in the form of deep neural networks. 
Advances in image models have translated into increasingly better performance in the applications they are build for, with trained deep networks often outperforming their competitors 
for tasks ranging from compression over denoising to compressive sensing~\citep{toderici_variable_2015,agustsson_soft_2017,theis_lossy_2017,burger_image_2012,zhang_beyond_2017,bora_compressed_2017,heckel_deep_2018}.

Most image generating deep neural networks are convolutional neural networks. Examples include the generators in generative adversarial networks~\cite{goodfellow_generative_2014,radford_unsupervised_2015},  variational and traditional autoencoders~\cite{pu_variational_2016,hinton_reducing_2006},
as well as autoencoder like structures such as the U-net~\citep{ronneberger_u-net_2015}. 
All the aforementioned image generating convolutional neural networks consist of only few operations: Upsampling, convolutions, and application of non-linearities. 

For solving inverse problems 
convolutional networks are typically trained on large datasets. A work by Ulyanov et al.~\cite{ulyanov_deep_2018}, however, has shown that overparameterized convolutional deep networks of autoencoder architecture enable solving denoising, inpainting, and super-resolution problems well even without any training, by fitting the weights of the network to a single image. Subsequently Veen et al.~\cite{veen_compressed_2018} have demonstrated that this approach also enables solving compressive sensing problems. 
However, since the network is highly overparameterized, this technique critically relies on regularization by i) early stopping and ii) adding noise to the input of the network during optimization.

Later that year, the paper~\cite{heckel_deep_2019} proposed a simple image model, called the deep decoder, that, in contrast to the network in the papers~\cite{ulyanov_deep_2018,veen_compressed_2018} is underparameterized and can therefore both compress images as well as regularize inverse problems, without any further regularization in the form of early stopping or the alike. 
The deep decoder only consists of upsampling operations and linearly combining channels, and does not use parameterized convolutions as in a conventional neural network. However its structure is closely related to a convolutional network in that it uses pixelwise linear combinations of channels, also known as 1x1 convolutions. 

In this work, we build on those findings by studying a variety of convolutional generators (which can be viewed as variants of the deep decoder~\cite{heckel_deep_2019}) for recovering an image from few measurements, both in theory and practice.
Our key contributions are as follows:
\begin{itemize}
\item 
We start with studying convolutional generators as concise image generators, and find that architectures with layers that 
i) upsample and convolve with a fixed interpolation kernels (such as the deep decoder) 
ii) upsample and convolve with a parameterized convolution kernel (i.e., transposed convolutions) 
iii) only convolve with fixed interpolation kernels 
iv) only convolve with parameterized convolution kernels, 
all perform well for compressing an image into a concise set of network weights.

\item We provide a theoretical explanation why convolutional decoders enable concise image representations: natural images can be modeled as piecewise linear/smooth functions and a convolutional decoder can represent a piecewise linear function of $s$ pieces concisely, that is with $O(s)$ parameters.

\item We then focus on recovering an image from few measurement, known as compressive sensing,
and show theoretically (for all variants of generators mentioned above), on the order of the parameters of the network many measurements are sufficient for recovery.

\item 
Most importantly, we show that compressive sensing regularized with an un-trained convolutional generator enables 
reconstruction of an image from compressive Magnetic Resonance Imaging (MRI) measurements with better performance than traditional sparse recovery techniques such as $\ell_1$ and total variation minimization, on real data.


\end{itemize}


\section{Convolutional generators}

We consider architectures that map a fixed tensor (typically chosen randomly) $\mB_1 = [\vb_{11}, \ldots, \vb_{1k}]\in \reals^{n_1 \times k}$ consisting of $k$ many $n_1$-dimensional channels to an $n_d \times \kout$ dimensional image, where $\kout=1$ for a grayscale image, and $\kout=3$ for an RGB image with three color channels. Throughout, $n_i$ has two dimension if the output of the network is an image, and one dimension if the output is a vector.
The network transforms the fixed input tensor to an image using only upsampling/no-upsampling and convolutional operations.
Specifically, the channels in the $(i+1)$-th layer are given by, for $i = 1, \ldots, d$,
\begin{align}
\label{eq:onelayerdecorig}
&\mB_{i+1}
= \cn\left( \relu
\left( \left[
\sum_{j=1}^k \mT(\vc_{i1j}) \mU_i \vb_{ij}, 
\ldots,
\sum_{j=1}^k \mT(\vc_{ikj}) \mU_i \vb_{ij}
\right] \right) \right). 
\end{align}
Here, $\mT(\vc)$ is the operator performing a convolution with the kernel $\vc \in \reals^{\ell}$, and $\mU_i$ is either the identity (no upsampling) or an upsampling operator that upsamples the signal by a factor of two. 
For example, for a one-dimensional signal, the upsampling operator applied to $\vx = [x_1,x_2,\dots, x_n]$ yields $\mU \vx  = [x_1,0,x_2,0,\dots, x_n,0]$.
Moreover, $\cn(\cdot)$ is a channel normalization operation, that, when applied to a channel $\vz_{ij}$ yields 
$
\vz'_{ij} = \frac{ \vz_{ij} - \empmean(\vz_{ij})  }{ \sqrt{ \empvar(\vz_{ij}) +\epsilon}} + \beta_{ij}
$,
where $\empmean$ and $\empvar$ compute the empirical mean and variance, and $\beta_{ij}$ is a parameter learned independently for each channel, and $\epsilon$ is a fixed small constant for numerical stability. 
The channel normalization operations is critical for optimizing convolutional generators~\cite{dai_channel_2019}.
Finally, the output of the $d$-layer network is formed as
\[
\img = \mathrm{sigmoid}\clr{(\mB_d \mC_{d})},
\]
where $\clr{\mC_{d}} \in \reals^{k \times \kout}$.
We consider the following architectures:
\begin{itemize}
\item 
{\bf i) Fixed interpolation and upsampling.}
This is the original deep decoder architecture~\citep{heckel_deep_2019}, which applies
bi-linear upsampling followed by linearly combining the channels with learnable coefficients.
This corresponds to choosing $\mU_d = \mI$, all other $\mU_i$ as upsampling operators, and the operators $\mT$ such that they convolve with the kernel
\[
\vc_{isj} 
=
1/16
\begin{bmatrix}
1 & 3 & 3 & 1 \\
3 & 9 & 9 & 3 \\
3 & 9 & 9 & 3 \\
1 & 3 & 3 & 1
\end{bmatrix} c_{isj},
\]
where $c_{isj} \in \reals$ is a learnable parameter.
\item 
{\bf ii) Parameterized convolutions and upsampling.}
This is equivalent to applying transposed convolutional layers with learnable filters; specifically $\mU_d = \mI$, all other $\mU_i$ as upsampling operators, and choosing the convolutional kernels as $4\times 4$ learnable filters.
\item 
{\bf iii) Fixed interpolation kernels and no upsampling.}
Same fixed convolutional kernels as in architecture i), but no upsampling, i.e., $\mU_i = \mI$ for all $i$.
This network convolves with fixed convolutional kernels, linearly combines channels, and applies non-linearities.
\item 
{\bf iv) Parameterized interpolation kernels and no upsampling.} Same as the deconvolution decoder iii), but without upsampling, i.e., $\mU_i = \mI$ for all $i$.
\end{itemize}

\section{Image representations with convolutional generators}

We first show that each convolutional architecture i-iv is capable of representing an image concisely with few parameters, demonstrating that convolutional generators---with both learned or fixed convolutions enable concise image representations, without any training! This shows that the results in~\citep{heckel_deep_2019} extend to a broader class of networks than architecture i. We then provide a potential theoretical explanation by showing that convolutional generators can represent piecewise linear function with few coefficients; and natural images are approximately piecewise linear/smooth. 

\subsection{Convolutional generators enable concise image representations
\label{sec:convgen}
}

The compression performance of convolutional generators is on par with state-of-the-art wavelet thresholding. 
To demonstrate this, we take architectures i-iv with $d=6$ layers and output dimension $512\times 512 \times 3$, and choose the number of channels, $k$, such that the compression factor, defined as the output dimensionality ($3 \cdot 512^2$) divided by the number of parameters of the network, $\Nparam$, is 32 and 8, respectively. 
We then take 100 randomly chosen images from the ImageNet validation set, 
and for each image $\vx^\ast$ and architecture fit the networks weights $\mC$ by minimizing the loss
$
\loss(\param,\vx^\ast) = \norm[2]{\prior(\param) - \img^\ast }^2
$
using the Adam optimizer. 
We then compute for each image the corresponding peak-signal-to-noise ratio (PSNR), 
and compare compression performance to wavelet compression~\citep{antonini_image_1992}, by representing each image with the $\Nparam$-largest wavelet coefficients. 
The results, depicted in Figure~\ref{fig:deepdeccompression}, show that each architecture has compression performance comparable to wavelet thresholding or better.
Wavelets are a strong baseline as they are one of the best methods to represent images with few coefficients.

The main takeaway from this experiment is that convolutional generators with both parameterized and fixed convolutions enable concise image models, even without any learning on a dataset. 
As we see next, forcing an image to lie in the range of such a generator enables regularization of inverse problems.

Note that architectures iii and iv are computationally inefficient since each channel is high-dimensional (512x512 dimensional in our example), and therefore in the following we only focus on architectures i and ii which are more efficient. 

\pgfplotstableset{
    create on use/D/.style={create col/copy column from table={./dat/imagenet_fit_psnrs_deconv16.dat}{0}}}
\pgfplotstableset{
    create on use/W/.style={create col/copy column from table={./dat/imagenet_fit_psnrs_wavelet_1_64.dat}{0}}}
\pgfplotstableset{
    create on use/DD/.style={create col/copy column from table={./dat/imagenet_fit_psnrs_up64.dat}{0}}}  
\pgfplotstableset{
    create on use/G/.style={create col/copy column from table={./dat/imagenet_fit_psnrs_gaussian_64.dat}{0}}}  
\pgfplotstableset{
    create on use/DNU/.style={create col/copy column from table={./dat/imagenet_fit_psnrs_3_k21.dat}{0}}}

\pgfplotstableset{
    create on use/D128/.style={create col/copy column from table={./dat/imagenet_fit_psnrs_deconv33.dat}{0}}}
\pgfplotstableset{
    create on use/W128/.style={create col/copy column from table={./dat/imagenet_fit_psnrs_wavelet_1_128.dat}{0}}}
\pgfplotstableset{
    create on use/DD128/.style={create col/copy column from table={./dat/imagenet_fit_psnrs_up128.dat}{0}}}
\pgfplotstableset{
    create on use/G128/.style={create col/copy column from table={./dat/imagenet_fit_psnrs_gaussian_128.dat}{0}}}
\pgfplotstableset{
    create on use/DNU128/.style={create col/copy column from table={./dat/imagenet_fit_psnrs_3_k43.dat}{0}}}


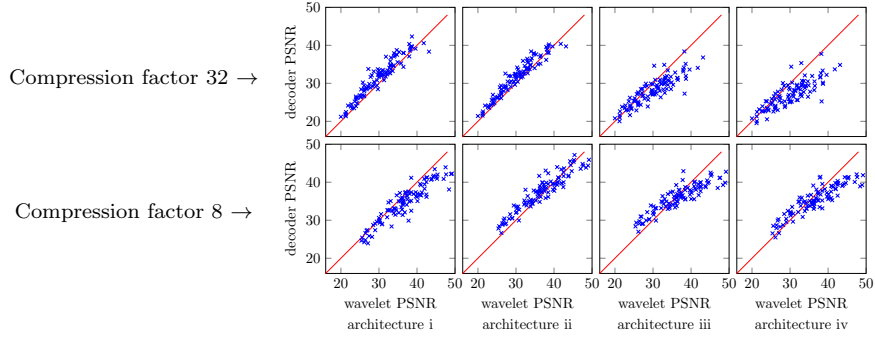
\begin{figure}
\begin{center}

\usetikzlibrary{fadings,shapes.arrows,shadows}   

\begin{tikzpicture}[scale=0.51]

\node at (-5,1.5) {\scriptsize Compression factor $32$ $\rightarrow$} ;
\node at (-5,-2) {\scriptsize Compression factor $8$ $\rightarrow$} ;

\begin{groupplot}[
legend style={at={(1,1)}},
         title style={at={(0.5,-1.8cm)}, anchor=south}, group
         style={group size= 4 by 2, xlabels at=edge bottom, ylabels at=edge left, yticklabels at=edge left,xticklabels at=edge bottom,
           horizontal sep=0.2cm, vertical sep=0.2cm}, xlabel={iteration}, ylabel={decoder PSNR},
         width=0.3\textwidth,height=0.3\textwidth, xmax=50, ymax=50,xmin=16,ymin=16]
         
         \nextgroupplot[
         ] 
	\addplot +[only marks,mark = x] table [x=W, y=DD] {./dat/imagenet_fit_psnrs.dat}; to 
	\addplot +[mark=none] coordinates
{(15,15)    (48,48)};

         \nextgroupplot[] 
	\addplot +[only marks,mark = x] table [x=W, y=D] {./dat/imagenet_fit_psnrs.dat};
	\addplot +[mark=none] coordinates
{(15,15)    (48,48)};

         \nextgroupplot[] 
	\addplot +[only marks,mark = x] table [x=W, y=G] {./dat/imagenet_fit_psnrs.dat};
	\addplot +[mark=none] coordinates
{(15,15)    (48,48)};

         \nextgroupplot[] 
	\addplot +[only marks,mark = x] table [x=W, y=DNU] {./dat/imagenet_fit_psnrs.dat};
	\addplot +[mark=none] coordinates
{(15,15)    (48,48)};

         \nextgroupplot[title={architecture i},xlabel={wavelet PSNR},ylabel={decoder PSNR}] 
	\addplot +[only marks,mark = x] table [x=W128, y=DD128] {./dat/imagenet_fit_psnrs.dat};
	\addplot +[mark=none] coordinates
{(15,15)    (48,48)};

         \nextgroupplot[title={architecture ii},xlabel={wavelet PSNR}] 
	\addplot +[only marks,mark = x] table [x=W128, y=D128] {./dat/imagenet_fit_psnrs.dat};
	\addplot +[mark=none] coordinates
{(15,15)    (48,48)};

         \nextgroupplot[title={architecture iii
         },xlabel={wavelet PSNR}] 
	\addplot +[only marks,mark = x] table [x=W128, y=DNU128] {./dat/imagenet_fit_psnrs.dat};
	\addplot +[mark=none] coordinates
{(15,15)    (48,48)};

         \nextgroupplot[title={architecture iv 
         }
         ,xlabel={wavelet PSNR}] 
	\addplot +[only marks,mark = x] table [x=W128, y=G128] {./dat/imagenet_fit_psnrs.dat};
	\addplot +[mark=none] coordinates
{(15,15)    (48,48)};

\end{groupplot}
\end{tikzpicture}
\end{center}
\caption{
\label{fig:deepdeccompression}
Representing images with convolutional generators: each dot is the PSNR obtained by representing an ImageNet image with the respective architecture as well as with the same number of wavelet coefficients. 
While all architectures are efficient at representing images, architecture ii (parameterized convolutions and upsampling) performs particularly well, closely followed by i (fixed convolutions and upsampling).
}
\end{figure}


\subsection{Representational capabilities of convolutional architectures}

Natural images are piecewise linear/smooth with sharp edges. 
We next show that a convolutional architecture is well suited for representing such functions by proving that it can 
represent a (discrete) piecewise linear function with $s$ pieces with $O(s)$ coefficients. 
Thus, by constraining the number of coefficients of a convolutional encoder to be small, we are enforcing a (piecewise) smooth signal at the output.


Consider a $d$-layer network with output $\reals^{n_d}$. 
We consider an architecture with linear upsampling operations truncated at the boundaries so that the operator 
$\mT(\vc) \mU_i 
=
\mM_i c
\colon \reals^{n_i}  \to \reals^{2 n_i -1}$
applied to a signal $\vb \in \reals^{n_i}$ becomes (for $n_i = 3$, as an example):
\[
\mM_2 \vb
= 
\frac{1}{2}
\begin{bmatrix}
2 & 1 & 0 & 0 & 0 \\
1 & 2 & 1 & 0 & 0 \\
0 & 1 & 2 & 1 & 0 \\
0 & 0 & 1 & 2 & 1 \\
0 & 0 & 0 & 1 & 2  
\end{bmatrix}
\begin{bmatrix}
b_1 \\
0 \\
b_2 \\
0 \\
b_3
\end{bmatrix}
= 
\begin{bmatrix}
b_1 \\
1/2(b_1+b_2) \\
b_2 \\
1/2(b_2+b_3) \\
b_3
\end{bmatrix}.
\]
We set the initial volume to 
$\mB_1 = \begin{bmatrix}
	1 & 0 \\
	0 & 1 
	\end{bmatrix}$, and suppose the network has a bias term, so that the $(i+1)$-st channel is given by
\[
\mB_{i+1}
=
\relu\left(
\mM_i \mB_i \mC_i + \vect{1} \transp{\va_i}
\right).
\]
Here, $\va_i$ is a vector containing biases that are added to each channel individually, and $\mC_i$ is a coefficient matrix associated with the $i$-th layer, as before.
Note that the bias term is included in the channel normalization in the original formulation~\eqref{eq:onelayerdecorig}. 
The output of the $d$-layer network is formed as 
$G(\mC, \va) = \mB_d \vc_d + a_d$, where $\vc_d$ are coefficients forming a linear combinations of the channels $\mB_d$, and $a_d$ is an additional bias term.
Here, $\mC = (\mC_1,\ldots, \mC_d)$ and $\va = (\va_1,\ldots,\va_d)$ are the coefficients of the network.
Note that with our choice of $\mB_1$, we have that $n_1 = 2$, and by our choice of $\mM_i$, we have $n_{i+1} = 2 n_i+1$, for all $i = 1, \ldots, d-1$. 

We consider approximating a discrete $s$-piecewise linear function, which are signals obtained by uniformly sampling a piecewise linear function consisting of $s$ pieces, see Figure~\ref{fig:representation} for an illustration. 

\begin{proposition}
\label{prop:piecewiselin}
Let $\vx \in \reals^{n_d}$ be a discrete $s$-piecewise linear function. Then there is a choice of at most $O(s)$ non-zero parameters $\mC,\va$ such that 
$G(\mC,\va) = \vx$. 
\end{proposition}

This proposition follows by noting that with at most $d+1$ many non-zero coefficients, we can represent a discrete rectangular linear function $\vb_d \in \reals^{n_d}$ with slope $\alpha$, that is zero until index $p$, defined as
\[
[\vb_{d} ]_i 
=
\begin{cases}
0 & i = 0, \ldots,p \\
\alpha (i-p) & i = p+1, \ldots, n_d
\end{cases},
\]
see Figure~\ref{fig:representation} for an illustration. 
To see that, note that with $\vc_1 = [0,\alpha]$ and $a_{d-1} = -\alpha p$ we have
\[
\vb_{d}
= 
\relu(a_{d-1} + \relu(\ldots \relu( \mM_2 \relu( \mM_1 \mB_1 \vc_1 ) ) ).
\]
This is a discrete rectangular linear function and only requires $d+1$ non-zero coefficients, as desired.
Forming a linear combination of $s$ such functions plus adding a bias term enables us to represent any discrete $s$-piecewise linear function (again, see Figure~\ref{fig:representation} for an illustration), which proves Proposition~\ref{prop:piecewiselin}.

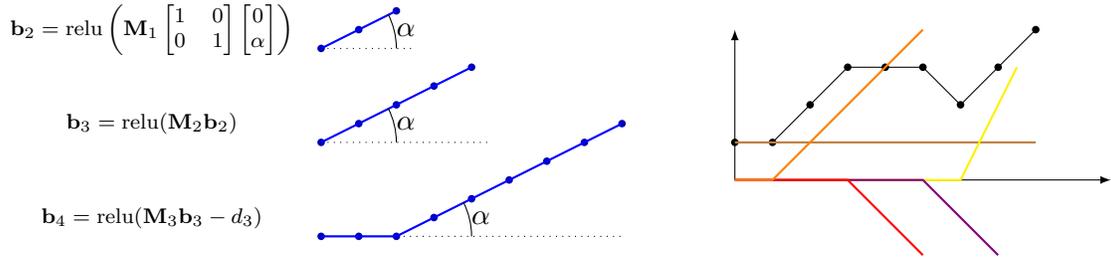
\begin{figure}
\begin{center}
\begin{tikzpicture}[>=latex,scale=0.5]

\begin{scope}
	\draw[dotted] (1,0) -- (3.5,0);
	\fill [DarkBlue] (1,0) circle (0.1);
	\fill [DarkBlue] (2,0.5) circle (0.1);
	\fill [DarkBlue] (3,1) circle (0.1);

	\draw[thick,blue] (1,0) -- (3,1);

	\draw (3,0) arc (0:27:2);
	\node at (3.25,0.5) {$\alpha$};
	
	\node at (-3.5,0.5) {\scriptsize $\vb_2 = \relu \left( \mM_1 \begin{bmatrix}
	1 & 0 \\
	0 & 1 
	\end{bmatrix}
	\begin{bmatrix}
	0 \\
	\alpha
	\end{bmatrix}
	 \right) $}; 
\end{scope}

\begin{scope}[yshift=-2.5cm]
	\draw[dotted] (1,0) -- (5.5,0);
	
	\fill [DarkBlue] (1,0) circle (0.1);
	\fill [DarkBlue] (2,0.5) circle (0.1);
	\fill [DarkBlue] (3,1) circle (0.1);
	\fill [DarkBlue] (4,1.5) circle (0.1);
	\fill [DarkBlue] (5,2) circle (0.1);

	\draw[thick,blue] (1,0) -- (5,2);

	\draw (3,0) arc (0:27:2);
	\node at (3.25,0.5) {$\alpha$};
	\node at (-3.5,0.5) {\scriptsize $\vb_3 = \relu(\mM_2 \vb_2)$};
\end{scope}

\begin{scope}[yshift=-5cm]
	\draw[dotted] (1,0) -- (9,0);

	\fill [DarkBlue] (1,0) circle (0.1);
	\fill [DarkBlue] (2,0) circle (0.1);
	\fill [DarkBlue] (3,0) circle (0.1);
	\fill [DarkBlue] (4,0.5) circle (0.1);
	\fill [DarkBlue] (5,1) circle (0.1);
	\fill [DarkBlue] (6,1.5) circle (0.1);
	\fill [DarkBlue] (7,2) circle (0.1);
	\fill [DarkBlue] (8,2.5) circle (0.1);
	\fill [DarkBlue] (9,3) circle (0.1);

	\draw[thick,blue] (1,0) -- (3,0) -- (9,3);

	\draw (5,0) arc (0:27:2);
	\node at (5.25,0.5) {$\alpha$};
	
	\node at (-3.5,0.5) {\scriptsize $\vb_4 = \relu(\mM_3 \vb_3 - d_3)$};

\end{scope}

\begin{scope}[xshift=11cm,yshift=-3.5cm]
\draw[->] (1,0) -- (11,0);
\draw[->] (1,0) -- (1,4);

\fill [] (1,1) circle (0.1);
\fill [] (2,1) circle (0.1);
\fill [] (3,2) circle (0.1);
\fill [] (4,3) circle (0.1);
\fill [] (5,3) circle (0.1);
\fill [] (6,3) circle (0.1);
\fill [] (7,2) circle (0.1);
\fill [] (8,3) circle (0.1);
\fill [] (9,4) circle (0.1);

\draw (1,1) -- (2,1) -- (4,3) -- (6,3) -- (7,2) -- (8,3) -- (9,4);

\draw[brown,thick](1,1) -- (9,1);
\draw[yellow,thick](1,0) -- (7,0) -- (8.5,3);
\draw[violet,thick](1,0) -- (6,0) -- (8,-2);
\draw[red,thick](1,0) -- (4,0) -- (6,-2);
\draw[orange,thick](1,0) -- (2,0) -- (6,4);
\end{scope}

\end{tikzpicture}
\end{center}
\caption{
\label{fig:representation}
{\bf Left:} Representation of a discrete rectangular function with a $d=4$ layer network.
{\bf Right:} An arbitrary discrete $s$-piecewise linear function (in black) can be approximated with a bias term (horizontal brown line) and a linear combination of $s$ many discrete rectangular functions.
}
\end{figure}


%
%
%
%
%
%



\section{Compressive sensing}

Compressive sensing is the problem of reconstructing an unknown signal $\vx^\ast \in \reals^n$ from $m < n$ linear, typically noisy, measurements
\[
\vy = \mA \vx^\ast + \noise,
\]
where $\mA \in \reals^{m\times n}$ is a known measurement matrix and $\noise \in \reals^m$ is unknown, additive noise.
In order to recover the signal $\vx^\ast$ from the measurement $\vy$, we have to make structural assumptions on the vector; the most common one is to assume that $\vx^\ast$ is sparse in some basis, for example in the wavelet basis. 
Assuming a sparse model works well for a number of imaging applications and is build on a solid theoretical foundation; specifically regularized $\ell_1$-norm minimization provably recovers $\vx^\ast$ from $\vy$ provided certain incoherence assumptions on the measurement matrix $\mA$ are satisfied~\citep{candes_robust_2006}.

More recent work assumes that the vector $\vx^\ast$ lies in the range of a generative prior, i.e., a neural network with fixed weights that were chosen by training on a large dataset of images, and demonstrates that this can perform better than standard $\ell_1$-minimization~\cite{bora_compressed_2017}. However, this relies on training a good generator for a class of images, and fails if there is a discrepancy of test and train images.

Here, we assume that $\vx^\ast$ lies in or close to the range of an un-trained convolutional model (a deep decoder). 
In order to recover the signal from the measurement $\vy$, we solve the optimization problem
\begin{align}
\label{eq:optproblem}
\hat \mC = \arg \min_{\mC} \norm[2]{ \vy - \mA G(\mC)}^2,
\end{align}
and estimate the unknown vector as $\hat \vx = G(\hat \mC)$.
Conceptually this approach is very similar to sparse recovery; here we optimize over $\vx$ in the range of a deep network, and in traditional compressive sensing approaches, optimization is over an $\ell_1$-norm ball.
We use the Adam optimizer to minimize the loss, but gradient descent works similarly well.
Due to the ReLU-nonlinearities, the optimization problem is non-convex and thus Adam or gradient descent might not reach a global optimum.

The advantage of this approach over recent deep learning based approaches 
that either learn an inverse mapping end-to-end~\cite{mousavi_learning_2017}
or assume that the signal lies in the range of a \emph{learned generative prior}~\cite{bora_compressed_2017} is that this approach does not require a pre-trained model. Therefore, our approach is suitable for applications in which little training data is available.


\subsection{
Recovery guarantees for compressive sensing 
\label{sec:recguarantees}
}

Under-parameterization provides a barrier to overfitting: As the next statement shows, an under-parameterized architecture enables recovery from a number of measurement that is on the order of the number of unknowns of the network. 

\begin{theorem}
\label{prop:maincsnew}
Consider a convolutional generator with $d$ layers, $N$ parameters, and input volume obeying $\norm{\mB_0} \leq \xi$. 
Consider a signal $\vx^\ast$, and a corresponding measurement
$
\vy = \mA \vx^\ast + \noise,
$
where $\mA \in \reals^{m \times n}$ is a measurement matrix with i.i.d.~Gaussian entries with zero mean and variance $1/m$. 
Let $\hat \mC$ be parameters that minimize $\norm{\vy - \mA G(\mC)}$ to within an additive factor of $\epsilon$ of the optimum over the ball 
$\setB(\mu) = \{\mC \in \reals^N \colon \norm[2]{\mC} \leq \mu\}$ for some $\mu > 0$
and suppose that the number of measurements obeys, for some slack parameter $\delta > 0$
\[
m  = \Omega( N d \log( d \xi \mu / \delta )).
\]
Then the estimate $G(\hat \mC)$ obeys
\[
\norm[2]{G(\hat \mC) - \vx^\ast}
\leq
6
\min_{\mC^\ast \in \mc B(\mu)} \norm[2]{G(\mC^\ast) - \vx^\ast}
+ 3 \norm[2]{\eta}
+ 2\epsilon
+ 2\delta.
\]
\end{theorem}

Theorem~\ref{prop:maincsnew}, proven in the appendix, follows from showing that a convolutional generator is Lipschitz, and combining this with results developed by Bora et al.~\citep{bora_compressed_2017}.

The statement guarantees that the number of measurements \emph{sufficient} for recovery is, up to logarithmic factors, on the order of the parameters of the generator network. This parallels results for sparse recovery which ensure that recovery is possible provided the number of measurements is, up to a logarithmic factor, on the order of the sparsity.
It is also related to the main result from~\cite{bora_compressed_2017} which ensures that recovery is possible provided the number of measurements exceeds a number that depends on the number of input parameters of a generative prior and the number of layers of the generative prior.

Of course, even if $\vx^\ast$ lies in the range of the generator $G$, it is not clear that minimization over the loss $\norm[2]{\vy - \mA  G(\mC) }$ or even the loss $\norm[2]{G(\mC) - \vx^\ast}$ yields a solution that is close to $\vx^\ast$ (i.e., $G(\mC) \approx \vx^\ast$), since the objective is non-convex. Thus, a result stating that the solution of an actual optimization scheme such as gradient descent has the properties stated in the theorem would be desirable, but we can currently only prove such a results for a shallow network.
However, our numerical results show that optimizing the objective with the Adam optimizer or gradient descent works very well, and more importantly, the deep decoder approach to compressive sensing slightly outperforms traditional compressive sensing recovery on real data.

Finally, we note that in concurrent work, the paper~\cite{Jagatap_Hegde_2019} established related results for a projected gradient descent algorithm. The projected gradient algorithms relies on solving, in each iteration, a non-convex problem of the form~\eqref{eq:optproblem} with $\mA=\mI$. Similar to our statement which holds for an algorithm that finds an $\epsilon$-accurate solution of the optimization problem~\eqref{eq:optproblem}, the paper~\cite{Jagatap_Hegde_2019} assumes that~\eqref{eq:optproblem} with $\mA=\mI$ can be solved sufficiently well.


\subsection{Generators with fixed convolutional filters provide smoother images} 

\label{sec:smoother}

While both architectures i and ii with fixed and learned convolutional filters enable concise image representations, fixed convolutional filters impose a stronger smoothness assumption and thus perform empirically better for compressive sensing with random measurement matrices.
To demonstrate this, we choose $\mA \in \reals^{m \times n}$ as a random measurement matrix with iid entries in $\{-1,+1\}$, and choose the undersampling factor as $n/m=3$. 
We estimate an image of dimension $128\times 128$ using architectures i (upsampling and fixed convolutions) and ii (upsampling and parameterized convolutions) as generative models, in both cases with about $3000$ parameters, well below the number of measurements $m = 5461 = 128^2/3$. As can be seen in Figure~\ref{fig:csex}, the decoder architecture with parameterized filters produces noise-like artifacts, while the architecture with fixed filters generates smoother images (due to the fixed upsampling kernels).
We found this effect to be even more pronounced with larger undersampling factors, and therefore focus on architecture i in the following.

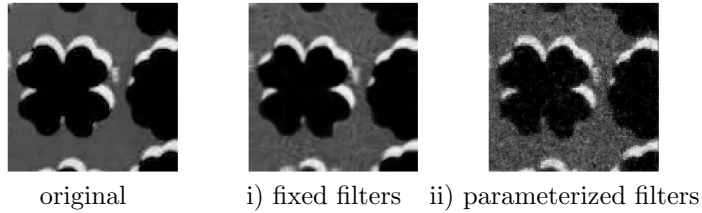
\begin{figure}
\begin{center}
\begin{tikzpicture}[scale = 1]

\newcommand\xspace{3.2cm}
\newcommand\yspace{3.3cm}
\newcommand\ymargin{1cm}
\newcommand\xmargin{1cm}
\newcommand\ycap{-4cm}
\newcommand\ycapp{-8cm}
\newcommand\iwidth{2.6cm}

\node at (-0.5*\xspace,-2*\yspace) {\includegraphics[width=\iwidth]{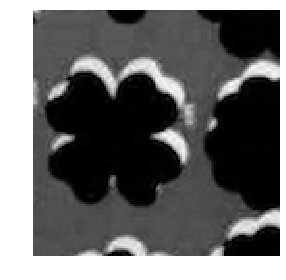}};
\node at (0.5*\xspace,-2*\yspace) {\includegraphics[width=\iwidth]{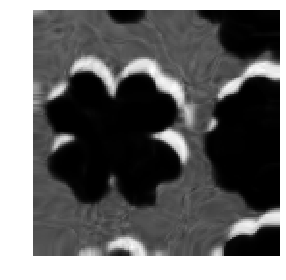}};
\node at (1.5*\xspace,-2*\yspace) {\includegraphics[width=\iwidth]{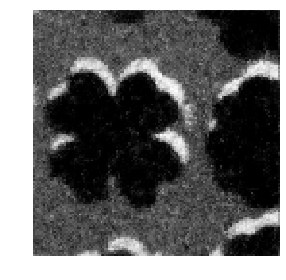}};

\node at (-0.5*\xspace,\ycapp) {\parbox{4cm}{\centering \small original}};
\node at (0.5*\xspace,\ycapp) {\parbox{4cm}{\centering \small i) fixed filters}};
\node at (1.5*\xspace,\ycapp) {\parbox{4cm}{\centering \small ii) parameterized filters}};

\end{tikzpicture}

\end{center}

\vspace{-0.5cm}

\caption{
\label{fig:csex}
Compressive sensing with a random measurement matrix and undersampling factor $n/m = 3$ for recovery of a $128\times 128$ image. 
The number of parameters of the decoder architectures i and ii are about $3000$, well below the ambient dimensions of the images, $128^2$.
}
\end{figure}


\subsection{Is the number of parameters a good measure on the number of required measurements?}

Theorem~\ref{prop:maincsnew} states that on the order of the number of parameters of a generator are sufficient for recovery, but is it also necessary?
Here we demonstrate that there are (random) signals in the range of the generator that require more measurements than parameters of the model for good estimation, indicating that in general the number of parameters is a good measure of the complexity of the model.
On the other hand, we also show that compressive sensing on real images works well even if the number of parameters of the decoder is larger than the number of measurements, indicating that convolutional generators regularize beyond what is explained by the number of parameters.


We focus on architecture i for the remainder of this paper, as we found it in Section~\ref{sec:smoother} to perform best. 
We generate random images in the range of the generator with varying number of parameters $N$, and recover them from $m$ linear random measurements. 
In order to generate an image, we can in principle simply choose its coefficients at random. However, 
this tends to generate `simple' images, in that a network with much fewer coefficients can represent them well. To ensure that we get `complex' or detailed images, we generate an image in the range of the generator by finding the best representation of noise for a fixed number of parameters of the model.
Figure~\ref{fig:plots}(a) shows the normalized mean square reconstruction errors for different choices of the number of parameters. As expected, for a larger number of parameters, $N$, the number of measurements $m$ needs to be larger for the recovery error to be small. 

Next, we consider recovery of natural images---we consider a detailed image and a simple image with little detail---and perform compressive sensing recovery with different number of parameters as well as with varying number of measurements. 
As expected, the simple image requires fewer measurements for successful reconstruction.
Also as expected, for a given number of measurements (take $m=200$),
if the number of parameters is too small, the approximation error is large and dominates; if the number of parameters is too large, the model provides little regularization and overfits. Thus, the best performance is obtained if the number of parameters is sufficiently small relative to $m$. 
Contrary to standard compressive sensing methods (such as sparse models), even when the model is overparameterized, recovery performs well. 
This indicates, that the recovery results in Section~\ref{sec:recguarantees} are very conservative when applied with natural images in mind.

\begin{figure}
\begin{center}
\begin{tikzpicture}[scale=0.95]

\node at (1.7cm,3.9cm) {\includegraphics[width=1.9cm]{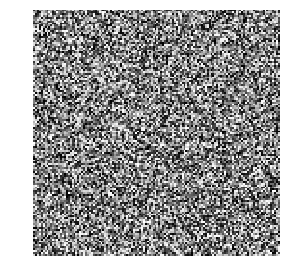}};

\node at (6.1cm,3.9cm) {\includegraphics[width=1.9cm]{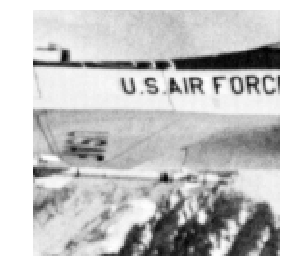}};

\node at (10.7cm,3.9cm) {\includegraphics[width=1.9cm]{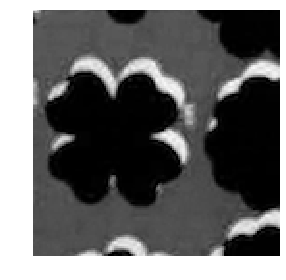}};

\begin{groupplot}[
xmode=log,
yticklabel style={
        /pgf/number format/fixed,
        /pgf/number format/precision=5
},
scaled y ticks=false,
         title style={at={(0.5,-0.45)},anchor=north},
         group style={group size=3 by 1, horizontal sep=1.0cm},
         width=0.31\textwidth, 
         legend style={at={(1.8,1)}},
         ]

\nextgroupplot[xlabel={$m$}, thick,ylabel={MSE},title={(a) rec. of rand. signal}]

\addplot +[mark=noe,red] table[x index=0,y index=1]{./dat/csrandimg.csv};

\addplot +[mark=noe,orange] table[x index=0,y index=2]{./dat/csrandimg.csv};

\addplot +[mark=noe,yellow] table[x index=0,y index=3]{./dat/csrandimg.csv};

\addplot +[mark=noe,green] table[x index=0,y index=4]{./dat/csrandimg.csv};

\addplot +[mark=noe,blue] table[x index=0,y index=5]{./dat/csrandimg.csv};


\nextgroupplot[xlabel={$m$}, thick,title={(b) CS with detailed img}]

\addplot +[mark=noe,red] table[x index=0,y index=1]{./dat/csf16img.csv};

\addplot +[mark=noe,orange] table[x index=0,y index=2]{./dat/csf16img.csv};

\addplot +[mark=noe,yellow] table[x index=0,y index=3]{./dat/csf16img.csv};

\addplot +[mark=noe,green] table[x index=0,y index=4]{./dat/csf16img.csv};

\addplot +[mark=noe,blue] table[x index=0,y index=5]{./dat/csf16img.csv};


\nextgroupplot[xlabel={$m$}, thick,title={(c) CS with simple img}]

\addplot +[mark=noe,red] table[x index=0,y index=1]{./dat/csf16img_poster.csv};
\addlegendentry{\scriptsize $N=410$}

\addplot +[mark=noe,orange] table[x index=0,y index=2]{./dat/csf16img_poster.csv};
\addlegendentry{\scriptsize $N=1620$}

\addplot +[mark=noe,yellow] table[x index=0,y index=3]{./dat/csf16img_poster.csv};
\addlegendentry{\scriptsize $N=3630$}

\addplot +[mark=noe,green] table[x index=0,y index=4]{./dat/csf16img_poster.csv};
\addlegendentry{\scriptsize $N=10050$}

\addplot +[mark=noe,blue] table[x index=0,y index=5]{./dat/csf16img_poster.csv};
\addlegendentry{\scriptsize $N=90250$}

\end{groupplot}
\end{tikzpicture}
\end{center}
\vspace{-0.6cm}
\caption{\label{fig:plots}
{\bf (a):} MSE for reconstruction of random images that lie in the range of an $N$-dimensional generator: As the theory suggests, the more measurements relative to $N$, the better reconstruction.
{\bf (b), (c):} Compressive sensing 
of a detailed (b) and a simple (c) image:
As expected, the simple image requires fewer measurements for successful reconstruction.
Also as expected, 
for a given number of measurements (take $m=200$),
if the number of parameters is too small, the approximation error is large; if it is too large, the model overfits, and the best performance is obtained if the number of parameters is sufficiently small relative to $m$. 
Contrary to standard compressive sensing results,  even when the model is overparameterized, recover performs well.
} 
\end{figure}
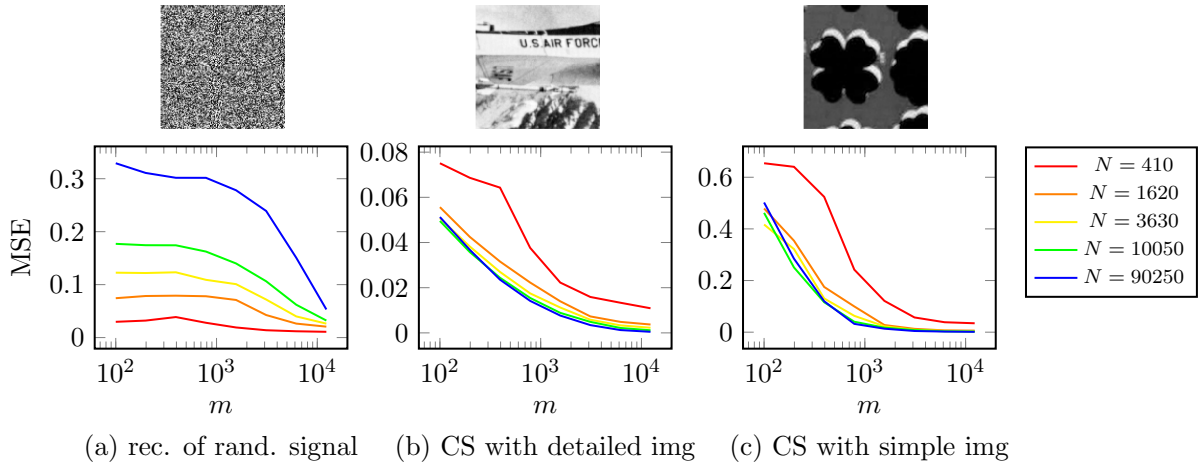


\subsection{Compressive sensing for MRI}

We next consider reconstructing an MRI image from few measurements. 
We focus on the architecture with fixed filters (i.e., the original architecture from~\cite{heckel_deep_2019}), since we have found this architecture to perform better for MRI reconstruction, compared to the architecture ii with parameterized filters. 

MRI is a medical imaging technique where magnetic fields are applied by a machine, and those fields induce the body part to be imaged to emit electromagnetic response fields that are measured by a receiver coil. 
Measurements correspond to points along a path through a two-dimensional Fourier space representation of the body part to be imaged, known as k-space. By taking a sequence of samples that tile the space up to some maximum frequency, an MRI machine can capture the full Fourier-space representation of a region, denoted by $\tilde \vy$. 
From the full Fourier-space respresentation, an image can be recovered by performing an inverse Fourier transform as $\vy = \inv{\mF} \tilde \vy$, where $\mF$ is the (2d) discrete Fourier transform matrix.
However, the number of samples captured in k-space is a limiting factor of MRI, and therefore it is common practice to accelerate the imaging process by undersampling the signal via omitting some of the samples. 
The problem is then to recover an image from the measurement $\mM \tilde \vy$, where $\mM$ corresponds to applying a mask in k-space, see Figure~\ref{fig:MRI} for an illustration of the mask corresponding to sub-sampling by a factor of $8$.

Thus, recovery from undersampled MRI measurements is a compressive sensing problem with the measurement matrix given by $\mA = \mM \mF$.
In order to evaluate the performance for this task, we consider the fastMRI dataset recently released by facebook~\cite{zbontar_fastMRI_2018}. Specifically, we consider the single coil validation dataset, and reconstruct images by regularizing with architecture i (the deep decoder) by solving~\eqref{eq:optproblem} using gradient descent.
We compare performance to recovery via least-squares as well as recovery with $\ell_1$-wavelet minimization and total variation (TV) regularization, both baseline compressive sensing methods.
See Figure~\ref{fig:MRI} for the corresponding results on an example image.
The results show that regularization with a un-trained convolutional network outperforms $\ell_1$-minimization and TV-minimization.
We repeated this experiment for 100 of the validation images from~\cite{zbontar_fastMRI_2018} and found an improvement of about 1dB on average over all images (see table~\ref{tab:comparison} below). In addition, it can be seen that the images reconstructed by the convolutional network look sharper than even the least squares reconstruction from the full measurements.

\begin{table}[h!]
\begin{center}
\begin{tabular}{ll*{10}{c}r}
deep decoder k=128 & 34.51 dB \\
deep decoder k=256 & 34.74 dB \\
deep decoder k=512 & 34.60 dB \\
$\ell_1$-minimization & 33.20 dB \\
total variation & 33.37 dB
\end{tabular}
\end{center}
\caption{
\label{tab:comparison}
Performance comparison of decoder architectures for MRI reconstruction (performance is averaged over 100 MRI images).
Slight overparameterization ($k=256$) but not too much works best for this task and outperforms $\ell_1$ and  TV-norm regularization.
}
\end{table}

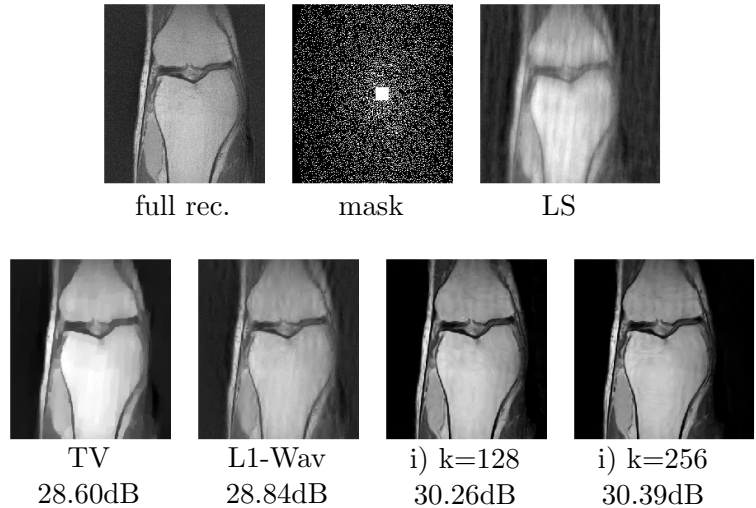
\begin{figure}
\begin{center}
\begin{tikzpicture}

\newcommand\xspace{2.5cm}
\newcommand\yspace{3.4cm}
\newcommand\ymargin{1cm}
\newcommand\xmargin{1cm}
\newcommand\ycap{-4.8cm}
\newcommand\ycapp{-8.4cm}
\newcommand\iwidth{2.8cm}

\node at (0*\xspace,-1*\yspace) {\includegraphics[width=\iwidth]{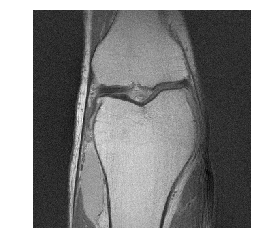}};
\node at (1*\xspace,-1*\yspace) {\includegraphics[width=\iwidth]{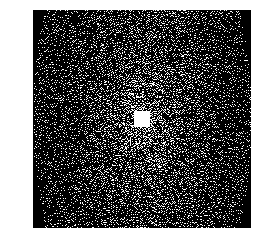}};
\node at (2*\xspace,-1*\yspace) {\includegraphics[width=\iwidth]{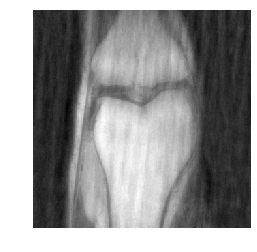}};
\node at (0*\xspace,\ycap) {full rec.};
\node at (1*\xspace,\ycap) {mask};
\node at (2*\xspace,\ycap) {LS};

\node at (-0.5*\xspace,-2*\yspace) {\includegraphics[width=\iwidth]{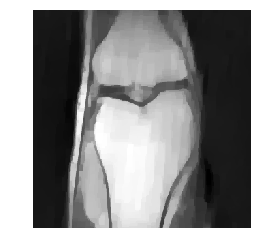}};
\node at (0.5*\xspace,-2*\yspace) {\includegraphics[width=\iwidth]{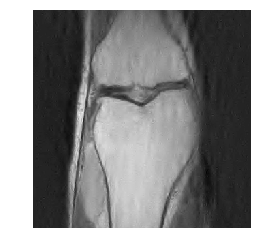}};
\node at (1.5*\xspace,-2*\yspace) {\includegraphics[width=\iwidth]{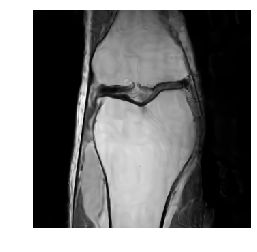}};
\node at (2.5*\xspace,-2*\yspace) {\includegraphics[width=\iwidth]{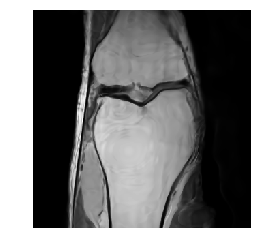}};

\node at (-0.5*\xspace,\ycapp) {\parbox{4cm}{\centering TV \\ 28.60dB}};
\node at (0.5*\xspace,\ycapp) {\parbox{4cm}{\centering L1-Wav \\ 28.84dB}};
\node at (1.5*\xspace,\ycapp) {\parbox{4cm}{\centering i) k=128 \\ 30.26dB}};
\node at (2.5*\xspace,\ycapp) {\parbox{4cm}{\centering i) k=256 \\ 30.39dB}};

\end{tikzpicture}
\end{center}
\vspace{-0.4cm}

\caption{
\label{fig:MRI}
An application architecture i for MRI reconstruction. 
Regularization with convolutional generators outperforms $\ell_1$-minimization and total-variation norm minimization, the standard reconstruction algorithms for untrained MRI reconstruction. 
Notably, the method does not overfit even if the model is overparameterized (i.e., when the number of channels, $k$, is large).
}
\end{figure}

\section*{Code}
Code to reproduce the experiments is available at \href{https://github.com/reinhardh/signalecovery\_decoder}{github.com/reinhardh/signal\_recovery\_decoder}.

\section*{Acknowledgements}
RH is partially supported by NSF award IIS-1816986, and would like to thank Paul Hand for helpful discussions on the approximation capabilities of the deep decoder.


\printbibliography




\appendix


\section{Proof of Theorem~\ref{prop:maincsnew}}

For notational simplicity, we prove the results for architecture i), for which the relation between layers becomes
\[
\mB_{i+1} 
=
\relu(\mU_i' \mB_i \mC_i).
\]
where the operator $\mU_i' =  \mT( (1/4) [1,2,1])\mU_i$ implements (scaled) linear upsampling.
The theorem is based

The proof is based on the following lemma by~\citep{bora_compressed_2017}. 

\begin{lemma}[{\citep[Thm.~1.2]{bora_compressed_2017}}]
\label{lem:dimakis}
Let $G\colon \reals^N \to \reals^n$ be $L$-Lipschitz and let 
$\setB(r) = \{\mC \in \reals^N \colon \norm[2]{\mC} \leq r\}$ be an $\ell_2$-norm ball.
Let $\mA \in \reals^{m \times N}$ be a random Gaussian matrix with i.i.d.~$\mc N(0,1/m)$ distributed entries and consider a noisy measurement $\vy = \mA \vx^\ast + \noise$, where $\vx^\ast \in \reals^n$ is a fixed signal. If $m = O(N \log( Lr / \delta ))$, then, with probability at least 
$1 - e^{-O(m)}$,
the parameters $\hat \mC$ that minimize $\norm{\vy - \mA G(\mC)}$ to within an additive $\epsilon$ of the optimum over the ball $\setB(r)$ obeys
\[
\norm[2]{G(\hat \mC) - \vx^\ast}
\leq
6
\min_{\mC^\ast \in \mc B(r)} \norm[2]{G(\mC^\ast) - \vx^\ast}
+ 2\epsilon
+ 2\delta.
\]
\end{lemma}

Consider the original deep decoder, where we have
\[
\mB_{i+1} 
=
\relu(\mU_i \mB_i \mC_i).
\]
Suppose the coefficients are bounded, specifically consider the set of $\mu$-bounded coefficients
\[
\setB_\mu
= 
\{\mC = \{\mC_0,\ldots,\mC_{d-1}, \vc_d\}
\in \reals^{k\times k} \times \ldots \times \reals^{k\times k} \times \reals^k
|
\norm[F]{\mC_i} \leq \mu \}.
\]
\begin{lemma}
\label{lem:lipschitz}
Consider a deep decoder with $\norm{\mB_0} \leq \xi$ and $\norm{\mU_i} \leq 1$.
On the set of $\mu$-bounded coefficient vectors, the deep decoder is $\xi \mu^dd$-Lipschitz, i.e., for all $\mC,\mC' \in B_\mu$, we have that
\[
\norm[2]{G(\mC) - G(\mC')}
\leq
\xi \mu^d
d
\left(
\sum_{i=0}^{d} 
\norm[F]{\mC_i' - \mC_i}^2
\right)^{1/2}.
\]
\end{lemma}

The lemma guarantees that the network is $\xi \mu^d d$-Lipschitz on the set of coefficient vectors with $\ell_2$-norm bounded by $\mu$. 
Application of this fact in Lemma~\ref{lem:dimakis} concludes the proof.

\begin{proof}[Proof of Lem.~\ref{lem:lipschitz}]
First note that
\begin{align}
\norm[2]{G(\mC) - G(\mC')}
&=
\norm[2]{ \mB_d \vc_d - \mB_d' \vc_d'} \nonumber \\
&\leq
\norm{\mB_d}
\norm[2]{\vc_d - \vc_d'}
+
\norm{\mB_d - \mB_d'}
\norm[2]{\vc_d}.
\label{eq:ineqfirst}
\end{align}
We start by upper bounding $\norm{\mB_d}$. 
With $\norm{\mU_i} \leq 1$, we have that
\[
\norm[F]{\mB_{i+1}}
\leq 
\norm[F]{\mU_i \mB_i \mC_i }
\leq
\norm[F]{ \mB_i }
\norm{\mC_i}
\leq \norm[F]{ \mB_i } \mu.
\]
This implies that
$\norm[F]{\mB_{i}} \leq \xi \mu^{i-1}$.

We next upper bound 
$\norm{\mB_d - \mB_d'}$. 
Towards this goal, as shown below we have that
\begin{align}
\label{eq:boundBBi}
\norm[F]{\mB_{i+1} - \mB_{i+1}' }
&=
\norm{\mB_i'} \norm[F]{\mC_i' - \mC_i}
+
\norm{\mC_i} \norm[F]{\mB_i' - \mB_i}.
\end{align}
Applying this inequality recursively, we obtain 
with $\norm{\mC_i} \leq \mu$ and $\norm{\mB_0} \leq \xi$ that
\begin{align}
\norm[F]{\mB_{d} - \mB_{d}' }
\leq
\xi \mu^{d-1}
\left(
\sum_{i=1}^{d-1} 
\norm[F]{\mC_i' - \mC_i}
\right)
\end{align}
Application in inequality~\eqref{eq:ineqfirst} yields 
\[
\norm[2]{G(\mC) - G(\mC')}
\leq
\xi \mu^d
\left(
\sum_{i=1}^{d} 
\norm[F]{\mC_i' - \mC_i}
\right).
\]
Finally, using that, for $\vx \in \reals^d$, $\norm[1]{\vx} \leq d \norm[2]{\vx}$ proves the statement. 

It remains to proof equation~\eqref{eq:boundBBi}:
\begin{align*}
\norm[F]{\mB_{i+1} - \mB_{i+1}' }
&=
\norm[F]{\relu(\mU_i \mB_i' \mC_i') - \relu(\mU_i \mB_i \mC_i) } \\
&\leq
\norm[F]{ \mU_i \mB_i' \mC_i' - \mU_i \mB_i \mC_i } \\
&=
\norm[F]{ \mU_i \mB_i' \mC_i'  
- \mU_i \mB_i' \mC_i
+ \mU_i \mB_i' \mC_i
 - \mU_i \mB_i \mC_i } \\
&\leq
\norm[F]{ \mU_i \mB_i' (\mC_i' - \mC_i)}
+
\norm[F]{\mU_i (\mB_i' - \mB_i) \mC_i} \\
&\leq
\norm{\mU_i}
\left( \norm{\mB_i'} \norm[F]{\mC_i' - \mC_i}
+
\norm{\mC_i} \norm[F]{\mB_i' - \mB_i}
\right).
\end{align*}

\end{proof}



\section{Additional recovery results
}

In this section, we prove additional recovery results for the special case of a one-layer network that are based on a subspace counting argument, as opposed to a Lipschitz-function based argument.

We consider a one-layer network. For simplicity, we consider a one-dimensional version, and ignore the normalization operation. Then, the networks output is given by
\begin{align}
\label{eq:onelayerdec}
G(\mC)
= 
\relu
\left( \left[
\sum_{j=1}^k \mT(\vc_{11j}) \mU_1 \vb_{1j}, 
\ldots,
\sum_{j=1}^k \mT(\vc_{1kj}) \mU_1 \vb_{1j}
\right] \right) \vc_2, 
\end{align}
where $\vc_{1ij} \in \reals^\ell$ are the convolutional filters and
$\mT(\vc)$ is the circulant matrix implementing the convolution operation.
For example for the case that $n = 5$ and $\ell=3$, the convolution matrix is given by
\begin{align}
\label{eq:Tc}
\mT(\vc) 
= 
\begin{bmatrix}
c_1 &  c_2 & c_3 & 0 & 0 \\
0    & c_1 & c_2 & c_3 & 0 \\ 
0  & 0 &  c_1 & c_2 & c_3 \\ 
c_3 & 0  & 0 &  c_1 & c_2  \\ 
c_2 & c_3 & 0  & 0 &  c_1  
\end{bmatrix}.
\end{align}

The following statement is our main result for recovery from few measurements.

\begin{theorem}
\label{prop:maincs}
Consider an image $\vx^\ast$, and a corresponding measurement
$
\vy = \mA \vx^\ast + \noise,
$
where $\mA \in \reals^{m \times n}$ a Gaussian random projection matrix 
with iid Gaussian entries with zero mean and variance $1/m$ with 
\[
m  = 
\begin{cases}
\Omega( \ell k^2  \log(n)) & \text{ if the filters of $G$ are parameters} \\
\Omega( k^2  \log(n)) & \text{ if the filters of $G$ are fixed}.
\end{cases}
\]
Consider a deep decoder $G(\mC)$ with one layer (see equation~\eqref{eq:onelayerdec}), and let $\hat \mC$ minimize $\norm[2]{\vy - \mA  G(\mC) }$ over $\mC$ to within an additive $\epsilon$ of the optimum.
Then, with probability at least $1 - e^{-\Omega(m)}$ over the random projection matrix,
\[
\norm[2]{G(\hat \mC) - \vx^\ast}
\leq
6
\min_{\mC} \norm[2]{G(\mC) - \vx^\ast} + (3/2) \norm[2]{\noise} + 2 \epsilon.
\]
\end{theorem}

The proof makes use of a lemma from~\citep{bora_compressed_2017} which introduces a variation of the Restricted Eigenvalue Condition~(REC)  from the compressed sensing literature and connects it to optimization over a not necessarily convex set of vectors.
A crucial difference of our setup and the one in~\citep{bora_compressed_2017}, is that in the latter setup, optimization is over the input of a neural network, whereas in our setup the input is fixed and we optimize over the weights of the network. 

\begin{definition}[{\citep[Def.~1]{bora_compressed_2017}}]
Let $\setS$ be a subset of point in $\reals^n$. 
A matrix $\mA$ is said to satisfy the SREC for the set $\setS$ and parameter $\gamma > 0$ if for all $\vx,\vx' \in \setS$,
\[
\norm[2]{\mA (\vx_1 - \vx_2)} \geq 
\gamma 
\norm[2]{\vx_1 - \vx_2}.
\]
\end{definition}

The following lemma gives the SREC an operational meaning, and follows by some algebraic manipulations from the definitions.

\begin{lemma}[{\citep[Lem.~4.3]{bora_compressed_2017}}]
Let $\mA$ be a random matrix satisfying the SREC$(\setS,\gamma)$ with probability $1-\delta$
and that obeys $\norm[2]{\mA \vx} \leq 2 \norm[2]{\vx}$, again with probability at least $1-\delta$.
For any $\vx^\ast$, let $\vy = \mA \vx^\ast + \noise$, and suppose $\hat \vx$ minimizes $\norm{\vy - \mA \vx}$ over $\vx \in \setS$ to within an additive $\epsilon$ of the optimum. 
Then,
\[
\norm[2]{\hat \vx - \vx^\ast}
\leq
\left( 4/\gamma + 1 \right)
\min_{\vx \in \setS} \norm[2]{\vx^\ast - \vx}
+ \frac{1}{\gamma} ( 2 \norm[2]{\noise} + \epsilon )
\]
with probability at least $1-2\delta$.
\end{lemma}

The proof is concluded by choosing $\gamma = 4/5$ and applying the following lemma with $\alpha = 1 - 4/5$.

\begin{lemma}
\label{eq:mainlemma}
Let $G$ be the one-layer decoder network in~\eqref{eq:onelayerdec}. 
And defined the set $\setS = \{ G(\mC) \colon \mC \in \reals^N \}$, where $N$ is the total number of parameters of the deep decoder, i.e., $N$. 
Let $\mA \in \reals^{m \times n}$ be a Gaussian random projection matrix with $m  = \Omega( N \log(n) / \alpha^2)$ where $\alpha \in (0,1)$ is some fixed constant.
Then $\mA$ satisfies the SREC$(\setS, 1-\alpha)$ with probability at least $1 - e^{-\Omega(\alpha^2 m)}$.
\end{lemma}

We conclude the providing a proof of lemma~\ref{eq:mainlemma}.


\subsection{Proof of Lemma~\ref{eq:mainlemma}}

We prove the result for the case where the filter weights are parameters, and thus the total number of parameters is given by $N = \ell k^2 + k$. The case where the filters are fixed, and thus the number of parameters is $N = k^2 + k$ is slightly simpler and follows in an analogous manner.

\begin{lemma}
\label{lem:1layerinsubspace}
Let $G$ be the one-layer decoder network in~\eqref{eq:onelayerdec}. 
Then $G$ lies in the union of at most $n^{\ell k^2}$ many $\ell k^2$-dimensional subspaces.
\end{lemma}

As a consequence of lemma~\ref{lem:1layerinsubspace}, the vector $\vx' = G(\mC_1) - G(\mC_2)$ lies in the union of at most $n^{2\ell k^2}$ many $2\ell k^2$ dimensional subspace. 
From standard results in compressive sensing (see for example~\citep[Thm.~9.9, Rem.~9.10]{foucart_mathematical_2013}, a Gaussian random matrix with i.i.d.~$\mc N(0,1/m)$ entries satisfies, for $\setU$ a $2\ell k^2$-dimensional subspace
\[
\PR{
\norm[2]{\mA \vx'} \geq (1-\alpha) \norm[2]{\vx'}, \text{ for all } \vx' \in \setU
} \geq 1 - e^{-\Omega(\alpha^2 m)}
\]
provided that $m = \Omega( \ell k^2 / \alpha^2 )$.
Taking the union bound over all the $n^{2\ell k^2}$-dimensional subspace, we get that $\mA$ satisfies the SREC($\{ G(\mC) \colon \mC \}, 1-\alpha)$ with probability at least $1 - n^{2\ell k^2} e^{-\Omega(\alpha^2 m)}$. 
Rescaling $\alpha$, we can conclude that 
$\mA$ satisfies the SREC($\{ G(\mC) \colon \mC \}, 1-\alpha)$ with probability at least $1 - e^{-\Omega(\alpha^2 m)}$ provided that $m = \Omega( \ell k^2 \log(n) / \alpha^2)$, which concludes the proof.

\begin{proof}[Proof of Lemma~\ref{lem:1layerinsubspace}]
We start by re-writing~\eqref{eq:onelayerdec} in a more convenient form.
First observe that we can write
\[
\mT(\vc) \mU_1 \vb 
=
\mH(\vb) \vc,
\]
where $\mH( \mU_1\vb) \in \reals^{n\times \ell}$ are the first $\ell$ columns of a Hankel matrix with first column equal to $\mU_1\vb$. To see this, note that for the example convolution matrix in equation~\eqref{eq:Tc}, the Hankel matrix becomes
\[
\mH(\vb)
=
\begin{bmatrix}
b_1 &  b_2 & b_3 \\
b_2 & b_3 & b_4 \\ 
b_3 & b_4 & b_5 \\ 
b_4 & b_5 & b_1 \\ 
b_5 & b_1 & b_2 
\end{bmatrix}.
\]
With this notation,
\begin{align*}
&G(\mC)\\
&=
\relu
\left( \left[ 
\sum_{j=1}^k \mH(\mU_1\vb_{1j}) \vc_{11j}, 
\ldots,
\sum_{j=1}^k \mH(\mU_1\vb_{1j}) \vc_{1kj}
\right] \right) \\
&=
\mU_1 \relu
\left( \left[ \mH(\mB_1)\vc_{11},
\ldots,
\mH(\mB_1)\vc_{1k}
\right] \right),
\end{align*}
where we defined
\[
\mH(\mB_1) = [\mH(\mU_1\vb_{11}), \ldots, \mH(\mU_1\vb_{1k}) ] \in \reals^{n \times k\ell}
\]
and $\vc_{1i} = \transp{[ \transp{\vc}_{1i1}, \ldots, \transp{\vc}_{1ik} ]} \in \reals^{k\ell}$.
For a given vector $\vx \in \reals^n$, denote by $\diag(\vx > 0)$ the matrix that contains one on its diagonal if the respective entry of $\vx$ is positive and zero otherwise. 
Denote by $\mW_{ji} \in \{0,1\}^{k\ell \times k\ell}$ the corresponding diagonal matrix $\mW_{ji} = \diag(\mH(\mB_j) \vc_{ji}>0)$. 
With this notation, we can write
\begin{align*}
G(\mC)
&=
\mU_1 
\left[ \mW_{11}\mH(\mB_1)\vc_{11},
\ldots,
\mW_{1k}\mH(\mB_1)\vc_{1k}
\right].
\end{align*}
Thus, $G(\mC)$ lies in the union of at-most-$\ell k^2$-dimensional subspaces of $\reals^n$, where each subspace is determined by the matrices $\{ \mW_{1j} \}_{j=1}^k$. 
The number of those subspaces is bounded by $n^{\ell k^2}$. This follows from the fact that for each matrix $\mW_{1j}$, by Lemma~\ref{lem:signpatterns} below, the number of different matrices is bounded by $n^k$. Since there are $k$ matrices, the number of different sets of matrices is bounded by $n^{\ell k^2}$. 
\begin{lemma}
\label{lem:signpatterns}
For any $\mW \in \reals^{n\times k}$ and $k\geq 5$,
\[
| \{ \diag(\mW \vv > 0) \mW | \vv \in \reals^k \} |
\leq 
n^{k}.
\]
\end{lemma}

\end{proof}





%


\end{document}